\documentclass[journal]{IEEEtran}

\usepackage{microtype}
\usepackage{graphicx}
\usepackage{subfigure}
\usepackage{booktabs} 
\usepackage{graphicx,subfigure}
\usepackage{amsfonts,amsmath,amssymb, mathrsfs, amsthm}
\usepackage{multirow}
\usepackage{epic,eepic,eepicemu}
\usepackage{epsf}
\usepackage{epsfig}
\usepackage{graphics}
\usepackage{mathrsfs}
\usepackage{psfrag}
\usepackage{tikz}
\usepackage{booktabs}
\usepackage{amsmath,amssymb}
\usetikzlibrary{arrows}
\usepackage{bbold}
\usepackage{url}

\DeclareMathOperator*{\argmax}{arg\,max}
\usetikzlibrary{arrows}
\usepackage{verbatim}

\usepackage{bm}

\newtheorem{proposition}{Proposition}
\newtheorem{example}{Example}

\newcommand*{\rom}[1]{\expandafter\@slowromancap\romannumeral #1@}
\newcommand\independent{\protect\mathpalette{\protect\independenT}{\perp}}
\def\independenT#1#2{\mathrel{\rlap{$#1#2$}\mkern2mu{#1#2}}}
\newcommand{\be}{\begin{equation}}
\newcommand{\ee}{\end{equation}}
\newcommand{\Ac}{\mathcal{A}}

\newcommand{\Dc}{\mathcal{D}}

\newcommand{\Xc}{\mathcal{X}}

\newcommand{\Zb}{\bm{Z}}
\newcommand{\Xb}{\bm{X}}
\newcommand{\xb}{\bm{x}}
\newcommand{\zb}{\bm{z}}
\newcommand{\eb}{\bm{e}}
\newcommand{\sbb}{\bm{s}}
\DeclareMathOperator{\E}{\mathbb{E}}
\newcommand{\xt}{\tilde{\bm{x}}}

\newcommand{\Xt}{\tilde{\bm{X}}}

\usepackage{color, soul}

\ifodd 1 \newcommand{\rev}[1]{{\color{black}#1}}
\else \newcommand{\rev}[1]{#1} \fi

\newcommand{\mahed}[1]{\sethlcolor{green}} 

\makeatother

\usepackage[skip=2pt,font=footnotesize]{caption}
\usepackage{hyperref}

\begin{document}

\title{Conservative Policy Construction Using Variational Autoencoders for Logged Data with Missing Values}

\author{Mahed Abroshan, Kai Hou Yip, Cem Tekin, Senior Member, IEEE, and Mihaela van der Schaar, Fellow, IEEE\thanks{M. Abroshan is with the Alan Turing Institute, London, UK, (mabroshan@turing.ac.uk). K. Yip is with University College London (kai.yip.13@ucl.ac.uk), C. Tekin is with Bilkent University (cemtekin@ee.bilkent.edu.tr), and M. van der Schaar is with University of Cambridge, Alan Turing Institute, and University of California Los Angeles, (mv472@cam.ac.uk).}
}

\maketitle

\begin{abstract}
In high-stakes applications of \rev{data-driven decision making} like healthcare, it is of paramount importance to learn a policy that maximizes the reward while avoiding potentially dangerous actions when there is uncertainty. 
\rev{There are two main challenges usually associated with this problem. Firstly, learning through online exploration is not possible due to the critical nature of such applications. Therefore, we need to resort to observational datasets with no counterfactuals. Secondly, such datasets are usually imperfect, additionally cursed with missing values in the attributes of features.}
In this paper, we consider the problem of constructing personalized policies using logged data when there are missing values in the attributes of features in both training and test data. The goal is to recommend an action (treatment) when $\Xt$, a degraded version of $\Xb$ with missing values, is observed. We consider three strategies for dealing with missingness. In particular, we introduce the \textit{conservative strategy} where the policy is designed to safely handle the uncertainty due to missingness. In order to implement this strategy we need to estimate posterior distribution $p(\Xb|\Xt)$, we use variational autoencoder to achieve this. In particular, our method is based on partial variational autoencoders (PVAE) which are designed to capture the underlying structure of features with missing values.\\

\textit{Index Terms}---Missing values, observational data, policy construction, variational autoencoder.  
\end{abstract}

\section{Introduction}

In many real-life applications, the datasets suffer from various forms of imperfection. Missingness in the attributes of the features is one of the most common types of imperfection \cite{little2012}. In the problem of constructing policies when there are missing values, one can simply use an imputation method to fill out missing attributes and then use one of the many existing approaches for policy recommendation for the complete dataset. However, this does not reflect the uncertainty in the features. Multiple imputations \cite{rubin2004} can be used instead of single imputation. In order to combine the recommended actions of different imputed instances, one simple idea is to use the mode of actions, another possibility is to use a stochastic policy where the probability of choosing an action is proportionate to its frequency. In this work, we address this problem in a systematic way. We suggest using a generative model, partial VAEs (PVAE) \cite{ma2018eddi}, to estimate the probability of different imputed features and use these probabilities as the scores of recommended actions for each particular complete feature. An advantage of using VAEs is that they make weak assumptions about the way the data is generated \cite{kingma2014,rezende2014}. Also, it has been shown that they are very effective in capturing the latent structure and the correlations among variables in several tasks \cite{louizos2017causal,gregor2015draw,ma2018eddi,rezende2016unsupervised}. Using posterior probabilities produced by PVAE, we can estimate the action that maximizes the expected reward. However, simply maximizing expected reward, given that we have uncertainty about the true feature, might be problematic in sensitive applications like healthcare, since the chosen action that maximizes expected reward may impose poor reward in some of the less likely scenarios which is not safe. To address this, we suggest using \textit{conservative strategy} for policy recommendations. With this strategy, we consider all likely scenarios (we can choose how prudent we need to be via a tuning parameter) and recommend an action that maximizes the reward in the worst-case scenario (a max-min criterion).

The main factor which differentiates the problem of learning from observational data from supervised learning is that for each feature the reward is only known for the prescribed action, i.e. we do not know the counterfactuals. Another complicating factor is that the logging policy (aka propensity score) is usually not random, hence we need to deal with the selection bias. In addition to these two issues, in this work, we are considering that features have missing attributes. Note that, as a consequence of this, we not only do not know counterfactuals but also no longer have access to the actual reward for a given action and complete feature. The goal of this work is to address how one can deal with the uncertainty imposed from the missing attributes. Note that there are other sources of uncertainty in the problem that we are leaving for future work. In particular, here we are using inverse propensity score (IPS) for dealing with selection bias---a method that is known to have high variance (especially when there are not enough samples for actions with low propensity scores) \cite{ionides2008truncated,batch2015}. Here, we are not considering this uncertainty and implicitly assume that there are enough samples to have a low variance estimate of the propensity score.

In summary, our contribution is as follows. We propose using a max-min criterion (conservative strategy) when there are missing values in the attributes for sensitive applications.  We are proposing a new method based on VAEs for handling missing attributes in the counterfactual estimation problem. In one of our methods we use the VAE to produce a similarity score to determine how much each of the samples should contribute to the estimation of the outcome for the sample in hand. In the other method, we use a conditional VAE setup to directly estimate the reward via the network. We are using the inverse propensity score for dealing with the selection bias.
\newline
\textbf{Notation:} We use capital letters for random variables, lowercase for realization, boldface letters for vectors, and calligraphic letters for denoting sets.
\section{Problem Definition and Related Work}\label{sec:problem}
The feature $\xb$ is a $d$-dimensional vector belonging to the set $\Xc=\Xc_1\times\Xc_2\times\cdots\times\Xc_d$. Here $\Xc_i$ can be a set of continuous, integer, or categorical variables. Define $\tilde \Xc_i= \Xc_i\cup \{*\}$, now the observed vector with missing attributes $\xt$ belongs to $\tilde \Xc=\tilde\Xc_1\times\tilde\Xc_2\times\cdots\times\tilde\Xc_d$. Define binary vector $\bm{M}$ which determines the missingness pattern. $M_i=0$ means that $i$th element is observed and $M_i=1$ otherwise. We assume missing at random (MAR) mechanism for missingness. \rev{This means that the probability of a value to be missing may only depend on the observed data} (see \cite{rubin1976inference} for exact definition). 
For each observed covariate $\xt$, we can recommend an action $a\in \Ac$ where $\Ac$ is a finite set (note that we are not restricting actions to be binary). The reward $R$ given action $a$ and true feature $\xb$ is drawn from an unknown distribution $R \sim \Phi(R \vert \xb, a)$. We denote $\mathbb{E}[R|\xb,a]$ by $\theta(\xb,a)$. The available dataset are triplets of $(\Xt_i,A_i,R_i)$:
\be \Dc^n=\{(\Xt_1,A_1,R_1),\cdots,(\Xt_n,A_n,R_n)\},\ee
where actions $A_i$ are produced from an unknown logging policy $\pi_0(A|\Xt)$ (also called generalized propensity score). Note that we assume that the treatments in the dataset are administered by only observing $\Xt$, hence Fig. \ref{fig:causal} represents the causal graph that describes the problem. The conditional distribution of these variables are given in \eqref{eq:dist}. With some abuse of notation the joint distribution is written as $(\Xb,\Xt,A,R)\sim p(\Xb,\Xt,A,R)$.
\begin{align}
\begin{split}
\Xb\sim \mu(\Xb), &\hspace{2mm} \Xt\sim p(\Xt|\Xb), \\ 
A\sim \pi_0(A|\Xt), &\hspace{2mm} R\sim \Phi(R|\Xb,A). \label{eq:dist}
\end{split}
\end{align}

\begin{figure}
\centering
\begin{tikzpicture}[scale=0.8]
\draw (2.5,4) node[rectangle,draw](X){$\Xb$}; 
\draw (0,4) node[rectangle,minimum width=0.5cm,draw](Y){$\Xt$};
\draw (0,1) node[rectangle,draw](A){A};
\draw (2.5,1) node[rectangle,draw](R){R};
\draw (4,2.5) node[rectangle,draw,text centered](R1){\footnotesize{$\{R(a)\}_{a\in \mathcal A}$}};
\draw (-2.5,4) node[rectangle,draw](M){M};

\draw [->] (X) edge (Y) (Y) edge (A) (A) edge (R) (X) edge[bend left] (R1) (R1) edge[bend left] (R) (M) edge (Y);

\end{tikzpicture} 
\vspace{0.3cm}
\caption{The causal model for noisy observation problem} 
\label{fig:causal}
\end{figure}
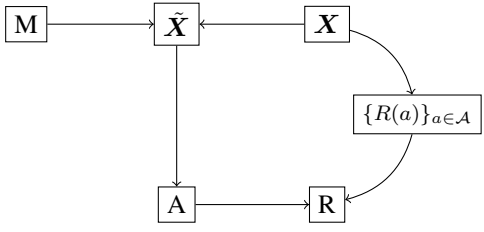
We make the two standard assumptions about the logging policy and rewards in the potential outcome framework \cite{rubin1974estimating,rubin2005causal}:
\begin{enumerate}
\item Common support: $\pi_0(a|\xt)>0$ for all $a\in \Ac$ and $\xt\in \tilde{\Xc}$.
\item Unconfoundedness with missing values: For each feature vector $\Xb$ the set of possible rewards $\{R(a)\}_{a\in \Ac}$ are statistically independent of the taken action: $\{R(a)\}_{a\in \Ac}\independent A\vert \Xt$.
\end{enumerate}
Note that the second assumption can be inferred from our causal graph. This is  required for using generalized propensity score $\pi_0(A|\Xt)$ as we do in Section \ref{sec:methods} \cite{rosenbaum1984reducing}.
The missing data problem frequently arises in machine learning. A motivating example for our model is the following, assume a medical setting where at the time that the treatment was administered to the patient, some of the attributes were missing. This can happen for various reasons, for example, maybe because of the different practices across different hospitals (where some of the attributes are not recorded), lack of certain tests, or maybe emergency situations to name a few. Note that since the treatment was administered only by observing $\Xt$ the causal model of Fig. \ref{fig:causal} holds. 
\subsection{Related Work}
We are considering the problem of off-policy evaluation (also known as offline evaluation in bandit literature). Here, we are using the IPS reweighting method \cite{horvitz1952,central1983} to deal with selection bias. We are using IPS in a deep network model, from this point of view our work is mostly related to \cite{atan2019constructing,joachims2018deep}. Direct method is another method for counterfactual estimation where the goal is to learn a function, mapping pairs of actions and features to rewards \cite{prentice1976use}. Doubly robust method combines the former two approaches \cite{robins1994estimation,dudik2011doubly}. Note that none of these works consider missing attributes in the feature. In \cite{hoiles2016bounded}, authors provide an off-policy evaluation method based on the regression estimator given in \cite{li2015toward,yoon2016}, when there are missing pairs of action and feature in the dataset. However, they do not consider missing attributes in the feature.

The treatment effect estimation problem is another related line of work, where the goal is to find the causal effect of a certain intervention (treatment) on the population or on individuals. The missing value problem is discussed in this framework from early on \cite{rosenbaum1984reducing,little2019statistical}. The use of generalized propensity scores, computed via multiple imputations is suggested in \cite{d2000estimating}. A summary of several early works can be found in \cite{hill2004reducing}. More recently, in \cite{kallus2018causal} matrix factorization has been proposed for estimating confounders from noisy covariates (also includes covariates with missing values). In \cite{mayer2020doubly} a doubly robust based method is suggested. In \cite{parbhoo2020information}, authors consider missing values only during test time and suggest a method based on information bottleneck technique. Finally, \cite{mayer2020missdeepcausal} suggests a new method based on VAEs (adopted for missing values) which learns distribution of the latent confounder and hence assumes a weaker condition than unconfoundedness with missing values which is harder to justify. In the experiment section, we compare our results with several of these recent works. Some of the other notable works in the treatment effect literature are \cite{shalit2017,wager2018,ganite2018,louizos2017causal}. However, they do not consider missing values. Thus, we will compare our results with \cite{louizos2017causal}, by imputing the missing values to get complete features and train and use the algorithm on the complete feature.

The problem studied in this work can be considered as an offline version of contextual bandits problem \cite{atan2018global,wang2017optimal,swaminathan2017off}. There are several works in bandit literature (and more generally in reinforcement learning) that are related to conservative strategy, e.g. \cite{wu2016conservative,garcia2015comprehensive,srinivasan2020learning}. 
The goal in bandit literature is to minimize regret, and conservatism in this area means guaranteeing that we are not performing poorly in the process of achieving a low regret (exploration). 
This is fundamentally different from conservatism in our problem which is due to uncertainty in the feature. 

Our method is based on PVAE introduced in \cite{ma2018eddi}. A few other VAE based methods are also suggested for the imputation task  \cite{kingma2014,mattei2019miwae}.  
Methods based on PVAE has been suggested for other tasks. In \cite{ma2018hybrid}, they use PVAE for hybrid recommender system and in \cite{icebreaker} for element-wise training data acquisition. 

\section{Strategies for Finding the Best Action}\label{sec:strategies}
In this section, we discuss three different strategies for action recommendation when there is uncertainty (in our problem due to missing attributes) in the features. \rev{In the next section, we address all three of these strategies using two different methodologies.} 

Assume that for feature $\xb$, the action that gives the highest expected reward is denoted by $a(\xb)$,
\be a(\xb)= \argmax_a \theta(\xb,a).\ee 
Since we observe $\xt$, the degraded version of $\xb$, there is uncertainty in the the true value of $a(\xb)$. This uncertainty can be quantified using Shannon entropy $H(a(\Xb)|\Xt=\xt)$. The following proposition presents an expansion for this quantity. 
\begin{proposition}
The uncertainty in the best action $a(\Xb)$ when we observe $\Xt\sim p(\Xt|\Xb)$ can be expressed as follows:
\be H(a(\Xb)|\Xt)= H(a(\Xb)) - (I(\Xb;\Xt) - I(\Xb;\Xt|a(\Xb)).\ee
\end{proposition}
\begin{proof}
See the proof in the appendix.
\end{proof}
The first term of the right hand side, $H(a(\Xb))$, represents the uncertainty in $a(\Xb)$ itself. This can be interpreted as the complexity of the function $a(.)$. For example, in the extreme case when there is a single action which is always the best action, then $H(a(\Xb))=0$. The second term $I(\Xb;\Xt)$ is the mutual information between $\Xb$ and $\Xt$ which represents the quality of the channel between these two variables, this channel is characterized by the conditional distribution $p(\Xt|\Xb)$, i.e., $I(\Xb;\Xt)$ shows how much information is passed to $\Xt$ from $\Xb$. The last term is subtracting the amount of information passed to $\Xt$ that is irrelevant to $a(\Xb)$. The probability that the best algorithm find $a(\Xb)$ by observing $\Xt$ is given by $\frac{1}{2^{H(a(\Xb)|\Xt)}}$. A simple example is given in appendix for which we compute these quantities, we leave further discussion about fundamental limits to future work. We reiterate that in this work we ignore the uncertainty in the true value of reward, i.e., the uncertainty in the estimation of $\theta(\xb,a)$. \rev{The above analysis holds for any degradation of the input. In particular, in this work we are considering missingness. The three strategies below can be used to deal with this type of uncertainty.}

\textbf{Imputation:} The first strategy is to use an imputation algorithm in order to find the most likely feature $\xb$ given the observed incomplete feature $\xt$, i.e.
\be \hat \xb=\argmax_{\xb} p(\xb|\xt). \label{eq:max:x}\ee
Then the recommended action $a_I(\xt)$ can be found by maximizing the reward for $\hat \xb$.
\be a_I(\xt)=a(\hat{\xb})=\argmax_a \, \theta(\hat{\xb},a).\label{eq:imp:best} \ee

\textbf{Maximum expected reward:} The imputation strategy recommends the action only \rev{based on one possible complete feature. This does not account for the uncertainty in the true feature.} A natural way is to directly maximize the expected reward instead of finding a single potential complete feature. Assuming that attributes are discrete and $\vert\Xc\vert$ is finite (the summation below should be replaced with an integral if this is not the case), the expected reward when $\xt$ is observed for a given policy like $\pi(A|\Xt)$ can be computed:
\begin{align*}
\mathbb E_{\pi}(R(\xt))&=\sum_{a,\xb} \theta(\xb,a)\pi(a|\xt) p(\xb|\xt),\\
&=\sum_{a} \pi(a|\xt)\sum_{\xb}\theta(\xb,a) p(\xb|\xt).
\end{align*} 
The policy $\pi$ that maximizes this expectation is a deterministic policy that recommends $a_M(\xt)$ defined as follows:
\be a_M(\xt)=\argmax_a \sum_{\xb} \theta(\xb,a) p(\xb|\xt).\label{eq:best}\ee
Multiple imputation method (MI) is an approximation for this strategy, \rev{where we} consider several possible complete features and recommend an action which maximizes the average reward over the imputed samples. This method is widely used in the literature  (e.g. see \cite{hill2004reducing,mattei2009estimating,mayer2020missdeepcausal}).

\textbf{Conservative strategy:} \label{sec:conserv} In sensitive applications, the strategies presented above may not be acceptable, because in these applications we have to avoid less likely (but still possible) scenarios for which a very low reward is expected (e.g. death in healthcare application). To achieve this, we suggest a max-min criterion that recommends the action which maximizes the reward in the worst case scenario which is likely ``enough'':
\be a_C(\xt)=\argmax_a \hspace{1mm} \min_{\xb:p(\xb|\xt)>cp(\hat \xb|\xt)} \theta(\xb,a). \label{eq:conservative} \ee
Here the constant $0\leq c<1$ determines how prudent we want to be ($\hat x$ is defined in \eqref{eq:max:x}). If we choose $c=0$, we get the most conservative policy, where we essentially ignore observed input $\xt$, and choose the action which has the highest minimum reward for all inputs, while $c\to 1$ is equivalent to the imputation method.

\rev{\textbf{Remark. } If we define $R=\int_S p(\xb|\xt) d\xb$ where $S=\{\xb \, |\, \, p(\xb|\xt)< c p(\hat \xb|\xt)\}$, then $R$ is representing the risk of not considering the true feature in the process of recommending the action. When we choose $c=0$, we have $R=0$, and it increases with $c$. The parameter $c$ then can be thought of as a tuning parameter for this risk. As a proxy, we can model $p(\xb|\xt)$ with a multivariate Gaussian distribution, and can consider $\hat \xb$ to be the center of the distribution. Thus, we can compute this risk for a given $c$.} 
\section{Estimation Methods}\label{sec:methods}
In this section, we suggest two methods for counterfactual estimation. We show how we can implement the three strategies discussed in the previous section using these two methods. In the core of both methods, we use PVAE.
In the first method, we train the network using only $\xt$ as the input, which will produce a similarity score for two features. We will use this similarity score to estimate the reward (we call this method SPVAE). In the second method (called CPVAE), we train a conditional VAE \cite{sohn2015learning} using $\xt$ incomplete context, and the reward (conditioned on action), and we will use the network for both estimating $p(\xb|\xt)$, and also estimating the expected rewards $\theta(\xb,a)$.
\begin{figure*}[t] 
\centering
\begin{minipage}{.5\textwidth}
\centering
\begin{tikzpicture}
\draw [fill=white!93!blue](-1,2) rectangle (0,2.5);
\node at (-0.5,2.25){$x_{1}$};
\draw [fill=white!93!blue](-1,2.75) rectangle (0,3.25);
\node at (-0.5,3){$\eb_{1}$};
\draw[->] (1,2.75) -- (1.5,2.75);
\draw[->] (0,3) -- (.4,2.75);
\draw[->] (0,2.25) -- (.41,2.65);

\draw [fill=white!93!blue](.7,2.75) circle (.3cm);
\node at (.7,2.75){$\times$};
\draw [fill=white!70!blue,rounded corners=5pt](1.5,2.5) rectangle (2.5,3);
\node at (2,2.75){h};

\draw [fill=white!93!blue](.7,.25) circle (.3cm);
\node at (.7,.25){$\times$};
\draw [fill=white!93!blue](-1,-.5) rectangle (0,0);
\node at (-0.5,-0.27){$x_{\vert O\vert}$};
\draw [fill=white!93!blue](-1,0.25) rectangle (0,0.75);
\node at (-0.5,.5){$\eb_{\vert O\vert}$};
\draw[->] (1,.25) -- (1.5,.25);
\draw[->] (0,0.5) -- (.4,.25);
\draw[->]  (0,-.25)-- (.41,0.15);
\draw [fill=white!70!blue,rounded corners=5pt](1.5,0) rectangle (2.5,0.5);
\node at (2,0.25){h};

\draw[->] (2.5,2.75) -- (3,1.75);
\draw[->] (2.5,0.25) -- (3,1.25);

\draw [fill=white!93!blue](3.15,1.5) circle (.3cm);
\node at (3.15,1.5){g};
\draw[->] (3.45,1.5) -- (4,1.5);
\draw [fill=white!93!blue,rounded corners=2pt](4,0.5) rectangle (6,2.5);
\node at (5,1.65){$f$};
\node at (5,1.2){\small (NN encoder)};
\draw[->] (6,1.5) -- (6.5,1.5);
\node at (6.7,1.5){$\Zb$};

\draw [fill=black](1,2) circle (.06cm);
\draw [fill=black](1,1.5) circle (.06cm);
\draw [fill=black](1,1) circle (.06cm);

\end{tikzpicture}
\vspace{2mm}
\caption{Encoder of PVAE (PNP Setting).}
\label{fig:encoder}
\end{minipage}%
\begin{minipage}{.5\textwidth}
\centering
\begin{tikzpicture}[scale=0.65]
\draw [fill=white!93!blue,rounded corners=2pt](0,0.5) rectangle (3,3.5);
\node at (1.5,2.35){PVAE};
\node at (1.5,1.55){Encoder};
\draw[->] (-1,3.5) -- (0,3);
\node at (-1.3,3.6){$\Xt$};
\draw[->] (-1,2) -- (0,2);
\node at (-1.3,2){$A$};
\draw[->] (-1,0.5) -- (0,1);
\node at (-1.3,0.5){$\tilde{R}$};
\draw[->] (3,2.5) -- (3.5,2.5);
\node at (3.7,2.5){$\Zb$};
\draw[->] (4,2.5) -- (4.5,2.5);
\node at (3.7,1.5){$A$};
\draw[->] (4,1.5) -- (4.5,1.5);
\draw [fill=white!93!blue,rounded corners=2pt](4.5,0.5) rectangle (7.5,3.5);
\node at (6,2.5){\footnotesize $q(\Xb,R|\Zb,A)$};
\node at (6,1.5){(Decoder)};
\draw[->] (7.5,1) -- (8.5,1);
\node at (8.7,1){$R$};
\draw[->] (7.5,3) -- (8.5,3);
\node at (8.7,3){$\Xb$};
\end{tikzpicture}
\vspace{2mm}
\caption{CPVAE structure}
\label{fig:sub2}
\end{minipage}%
\end{figure*}

\subsection{PVAE}\label{sec:PVAE}
We will be using the encoder of partial VAE that was introduced in \cite{ma2018eddi}. In particular, we use the Pointnet Plus (PNP) setting. The structure of the encoder is represented in Fig. \ref{fig:encoder}. \rev{PVAE is designed to deal with the missingness in the input and its structure allows the input dimension to vary.} Assume that $x_{i_1},\cdots, x_{i_{\vert O\vert}}$ are the observed attributes of the feature, each observed attribute $x_{i_j}$ will be multiplied by an embedding vector $\eb_j$ that will represent the position of the observed attribute. Denote the element-wise multiplication of $\eb_j$ and $\xb_{i_j}$ by $\sbb_j=\xb_{i_j}*\eb_j$. Now $\sbb_j$'s will be fed to $h$, a shared neural net. Then there is a permutation invariant function $g$ (in our setup $g$ is a summation similar to \cite{ma2018eddi}) that maps $(h(s_1),\cdots,h(s_{\vert O\vert}))$ to $\mathbb{R}^k$ ($k$ is a hyperparameter). Finally, this $k$-dimensional latent variable will be fed to \rev{a fully connected network} $f()$. Therefore, we have $\Zb=f\left(g(h(s_1),\cdots,h(s_{\vert O\vert}))\right)$. We refer to \cite{ma2018eddi} for a more detailed discussion about the encoder. 

For the decoder of PVAE we use \rev{a fully connected network}. Inspired from the decoder in \rev{the} HI-VAE model \cite{nazabal2018handling}, we consider the following distributions for different type of variables and map $\Zb$ \rev{to the parameters of an appropriate distribution.} This will enable us to handle heterogeneous features of the context. For continuous variables we have $p(x_i|\Zb)=\mathcal{N}(\mu_i(\Zb),\sigma_i(\Zb))$, where $\mu_i(\Zb)$ and $\sigma_i(\Zb)$ are outputs of the neural network with input $\Zb$. 
For categorical attributes, we use one-hot encoding, the posterior distribution is given by a softmax $p(x_i=j|\Zb)=\frac{\text{exp}^{-s_j(\Zb)}}{\sum_{t=1}^m \text{exp}^{-s_t(\Zb)}}$, where $s_t(\Zb)$ is the output of the decoder corresponding to the $t$th category.
The loss function is similar to the ELBO used for training of PVAE in \cite{ma2018eddi}.
\begin{align}
\begin{split}
&\log p(\Xt)\geq \log p(\Xt)-D_{KL}\left(q(\Zb\vert\Xt)\vert\vert p(\Zb|\Xt)\right)\\
&=\E_{\Zb\sim q(\Zb|\Xt)}\left[\log p(\Xt|\Zb)\right]-D_{KL}\left(q(\Zb\vert\Xt)\vert\vert p(\Zb)\right).
\end{split}\label{eq:ELBO}
\end{align}
We consider normal distribution for $p(\Zb)=\mathcal{N}(0,1)$. Similar to \cite{ma2018eddi,nazabal2018handling} we assume the two following equations hold. The  first one states the independence of attributes given the latent variable, i.e.
\be p(\xb|\Zb)=\prod_{i=1}^d p(x_i|\Zb),\label{eq:ind}\ee
The second one states that all the information about unobserved attributes in $\xt$ is encoded into $\Zb$, i.e. if $\xb_M$ represents the set of missing attributes, then we have
\be p(\xb_M|\xt,\Zb)=p(\xb_M|\Zb).\label{eq:latent}\ee
\subsection{SPVAE} \label{sec:stat}
We suggest using \rev{the} following simple estimator for finding $\hat\theta(\xb,a)$:
\be \hat\theta(\xb,a)=\sum_{i=1}^N w_i\, \frac{\mathbb{1}[A_i=a]\, R_i}{\hat \pi_0(a|\Xt_i)}.\label{eq:stat}\ee
where $w_i=\frac{p(\xb|\Xt_i)}{\sum_{j=1}^N p(\xb|\Xt_j)}$ are similarity scores corresponding to each of the data samples, \rev{and $\hat \pi_0(a|\Xt_i)$ is the estimation of the propensity score. We explain how to compute $\hat \pi_0(a|\Xt_i)$ in the next subsection}. Essentially $w_i$ shows how much the reward of sample $\Xt_i$ is relevant for estimating the reward for $\xb$. The inverse propensity score term adjusts for the selection bias in the data.
For estimating $p(\xb|\Xt_i)$, we can feed $\Xt_i$ to PVAE, the output of the network gives the required posterior distribution. Recall that we assumed Gaussian distribution for the output of the VAE. Using \eqref{eq:ind} and \eqref{eq:latent} we have
\be p(\xb|\Xt_i)=\prod_{j=1}^d p(x_j|\Zb). \label{eq:weight}\ee
A variation of SPVAE method that computes $\hat\theta(\xb,a)$ in a slightly different way is proposed in the supplementary material.
\rev{ \textbf{Remark.} Notice that the summation in \eqref{eq:stat} may become computationally costly. If this is the case, one can randomly choose $M<N$ samples from the dataset and estimate $\hat\theta(\xb,a)$ only using those $M$ samples. The CPVAE method that we propose next will not have this issue, since it can estimate the reward with a singe forward pass through a network.} 
\subsection{Propensity Score}
For \rev{computing $\hat \pi_0(A|\Xt)$,} we first use a multiple imputation method to produce multiple complete datasets. Any standard multiple imputation method like \cite{2010mice} or \cite{yoon2018gain} can be used. Then we fit a standard multinominal logistic regression model similar to \cite{atan2019constructing} on the completed features. In the test time we average the propensity score over multiple imputations. This is a classical method that is well studied in the literature \cite{d2000estimating, hill2004reducing,mattei2009estimating}.  
(It is known that averaging the propensity score before performing casual inference gives better result \cite{mitra2016comparison}.) 
A more advanced method for estimating propensity scores with missing values is recently introduced in \cite{mayer2020doubly}. We leave exploring effect of using more advanced methods for the future work. 





\subsection{CPVAE}
In this section, we modify PVAE and use it as an \rev{end-to-end} network to produce an estimation of $\theta(\xb,a)$. We use conditional VAE and call this method CPVAE.  Firstly, the input of the CPVAE is different. During training, the input is a subset of rewards, actions, and observed attributes that we represent with $(\Xt_i,A_i,\tilde{R}_i)$ (by denoting the rewards with $\tilde{R}$ we highlight that they might be missing from the input of the network). The idea is that in the test time, the reward can be treated as a missing attribute of the input, i.e. the input \rev{at} test time will be the observed attributes and action $(\Xt,A,*)$. The decoder network attempts to reconstruct $(\Xb,R)$, hence it will produce an estimation for the reward (see Fig. \ref{fig:sub2}). This method has the advantage that the correlations among different attributes of feature, reward, and action are expected to be captured by the latent variable $\Zb$. Also in comparison with SPVAE for producing the estimated rewards, we do not need to compute the summation in \eqref{eq:stat} and we can get the reward with a single forward pass through the network. (Note that it is expected to have a better quality of imputation using outcome in the imputation process \cite{moons2006using,hill2004reducing}.) 
\subsubsection{Loss Function}
The standard ELBO loss function for CPVAE can be written as follows:
\begin{align}
\begin{split}
\log p(\Xt,\tilde{R}|A)&\geq \log p(\Xt,\tilde{R}|A)\\
&-D_{KL}\left(q(\Zb\vert\Xt,\tilde{R},A)\vert\vert p(\Zb|\Xt,\tilde{R},A)\right)\\
&=\E_{\Zb\sim q(\Zb|\Xt,\tilde{R},A)}\left[\log p(\Xt,\tilde{R}\vert\Zb,A)\right]\\
&-D_{KL}\left(q(\Zb\vert\Xt,\tilde{R},A)\vert\vert p(\Zb\vert A)\right).
\end{split}\label{eq:elbo_EE}
\end{align} 
In order to account for the selection bias in the loss function, we use IPS technique and write the final loss as:

\begin{align}
\begin{split}
\mathcal L=&\frac{1}{N}\sum_{i=1}^N \sum_{a\in \Ac} \Bigg(\log p(\Xt_i|\Zb,A)+\log p(R_i|\Zb,A)\\
&-D_{KL}\left(q(\Zb\vert\Xt,A,R)\vert\vert p(\Zb|A)\right)\Bigg)\frac{\mathbb{1}[A_i=a]}{\hat \pi_0(a|\Xt_i)}
\end{split}
\end{align} 
Notice that $\mathbb{E}[\mathcal L]$ is equal to the lower bound in \eqref{eq:elbo_EE}. The IPS term can also be interpreted as a method to force the autoencoder to learn rare action-feature pairs more carefully by penalizing the loss function.
\subsection{Implementing Strategies}
In this section, we explain how to implement three strategies using our two suggested methods.
\begin{itemize}
\item \textbf{Imputation:} For SPVAE, when $\xt$ is observed, we first find the output of PVAE to impute the missing attributes of $\xt$ (we do not change the values which are not missing). Denote the complete feature vector by $\xb$. We now use \eqref{eq:stat} to estimate $\hat{\theta}(\xb,a)$ for all possible actions $a$, and then recommend the action which maximizes $\hat{\theta}(\xb,a)$. For CPVAE we simply feed $\xt$ along with different actions to the network and choose the action with highest expected reward.
\item \textbf{Maximum Expected Reward:} For SPVAE we feed $\xt$ to PVAE, then sample $t$ times from $q(Z|\xt)$ ($t$ is a hyperparameter) to get $\zb_1,\cdots,\zb_t$. Denote the imputed output of the decoder network of these $t$ latent variables by $\xb_1,\cdots,\xb_t$. For all $a\in \Ac$ and $\xb_i$, we use equation \eqref{eq:stat} to compute $\hat\theta(\xb_i,a)$. Then, we recommend $a$ which maximizes the average reward of these $t$ inputs. We do similarly for CPVAE, the estimation of the $\hat\theta(\xb_i,a)$ will be done by observing output of CPVAE.  
\item \textbf{Conservative:} We need to compute the following expression for all $a\in \Ac$
$$\min_{\xb:p(\xb|\xt)>cp(\hat \xb|\xt)} \theta(\xb,a).$$
First we pass $\xt$ through PVAE to get $\hat x$ and the posterior distribution $p(\xb\vert \xt)$. We produce $u$ samples from the generator model \rev{through random sampling.} That is, we can randomly sample $u$ times from $P(\Zb)$ to get $\zb_1,\cdots,\zb_u$ (recall that $p(\Zb)=\mathcal{N}(0,1)$). Then output of the decoder gives us $u$ generated samples $\xb_1,\cdots,\xb_u$. Now using posterior distribution $p(\xb\vert \xt)$ we find samples which satisfy the constraints. Assume that $\mathcal S$ is the set of all indices $1\leq i\leq u$ which $p(\xb_i|\xt)$ satisfies the inequality constraint. Then for SPVAE, we compute $\min_{i\in S}\hat\theta(\xb_i,a)$ using \eqref{eq:stat} for all $a\in \Ac$ and recommend $a$ which maximizes this expression. For CPVAE, for all $a\in \Ac$ we pass $\vert \mathcal S \vert$ samples along with actions $a$ and compute the minimum reward for each action. Then we recommend the action with highest reward.
\end{itemize}
\section{Experiments}\label{sec:exp}
In this section, we evaluate our suggested methods using three experiments. First, we use MNIST dataset \cite{MNIST}, and frame the usual classification task for identifying handwritten digits in a logged bandit setup. \rev{Note that for policy recommendation problems, since counterfactuals are not available,} evaluating an algorithm is not directly possible\rev{. That} is why here, similar to many other works (e.g. see \cite{atan2019constructing}), we use a classification problem to evaluate our method. \rev{We use this dataset to highlight the differences between the three strategies discussed in the paper.}  
\rev{Secondly, we use IHDP dataset \cite{hill2011}, which is a widely used dataset in treatment effect literature, to compare \rev{the} predictive capability of our methods in the presence of missing value with other recent suggested methods.} We show that our methods outperform state-of-the-art methods for estimating average treatment effect in the presence of  missing values. \rev{Finally, to further evaluate our method, we use OhioT1DM dataset \cite{marling2018ohiot1dm}, a medical dataset that includes blood glucose measurements and insulin doses for numerous type 1 diabetes mellitus patients using insulin pump therapy. }

\begin{table}[t!]
\centering
\caption{Expected reward of different strategies for MNIST data with $50\%$ missing attributes and average number of instances of reward less Than $-7$. \vspace{1mm}}
\begin{tabular}{ccc} 
\toprule
  &  $R$ & Num. of inst. $R<-7$   \\
\midrule   
Imputation  &$-1.21 \pm 0.01$&$2.8 \pm 0.83$ \\
MER   & $-1.18 \pm 0.01$&$2.7 \pm 0.47$ \\
Cons $c=0.7$ &$-1.51 \pm 0.02$&$1.3 \pm 1.05$\\
Cons $c=0.1$ & $-2.10 \pm 0.02$&$0.2 \pm 0.04$ \\
Cons $c=0.001$ & $-2.58 \pm 0.01$ &$0.0 \pm 0.00$ \\
\bottomrule 
\end{tabular}
\label{MNIST-unbiased}
\vspace{-5pt}
\end{table}
\subsection{MNIST}
In the first experiment, we use the MNIST dataset. \rev{The goal of this experiment is to compare different strategies introduced in Section \ref{sec:strategies}. Thus, here we only implement CPVAE using the three strategies.} The complete feature has $784$ attributes, each one is a number between 0 to 255. We will erase a fixed percentage of pixels (50 percent in this experiment) from each image uniformly at random. The goal is to predict the correct label associated with the image, hence the set of actions is $\Ac=\{0,1,\cdots,9\}$. The reward is defined as a Gaussian, centered around the difference of the true label ($y_i$) and the predicted one (i.e. action $A_i$) $$R_i\sim \mathcal N(-\vert y_i- A_i\vert,0.1).$$ 
Note that this is different from the standard binary reward defined for classification task (i.e. $R=1$ if the predicted label is correct, and zero otherwise). The reason we choose this reward is to highlight the differences between the three strategies and the necessary compromises that need to be made in the face of uncertainty. For example, assume that we are considering an image which is $0$ with probability $0.7$ and $8$ with probability $0.3$. In this scenario, using the reward that we defined all three strategies are meaningful (i.e. you may choose 2 to avoid low probability) while with the binary reward, all three strategies coincide (all three recommend $a=0$). 
The mechanism for assigning actions to images for creating dataset is as follows. For images representing even numbers, $\pi_0(a|\Xt_i)=1/20$ for $0\leq a<5$ and $\pi_0(a|\Xt_i)=3/20$ for $5\leq a<10$. For odd images, $\pi_0(a|\Xt_i)=3/20$ for $0\leq a<5$, and $\pi_0(a|\Xt_i)=1/20$ for rest of the actions.

In Table \ref{MNIST-unbiased} the average reward of different strategies is reported. We are using CPVAE method in this experiment. The maximum expected strategy get the highest reward as expected, followed by the imputation strategy. It can be seen that, as we decrease parameter $c$, the expected reward decreases. In exchange, the number of instances for which we get a poor reward (here we considered reward less than $-7$) is decreasing with $c$.

In Fig. \ref{fig:dist}, we show the distribution of recommended action for three conservative strategies where we change tuning parameter $c$, from top to down $c=0.001$, $c=0.1$, and $c=0.7$. For the first figure with $c=0.001$ it can be seen that the method always chooses action $5$, which is the safest action. This action avoids losses more than $5$. Since $c$ is too small, images of different digits can pass the condition on \eqref{eq:conservative} and hence the best action would be $5$ (or $4$). \rev{It can be seen that, as we increase $c$, fewer images with random digits pass the constraint and, as a result, the distribution of actions} spread over different actions. The details of the experiment setup and some additional experiments are available in the supplementary materials.
\begin{figure}[t]
\centering
\includegraphics[width=0.4\textwidth]{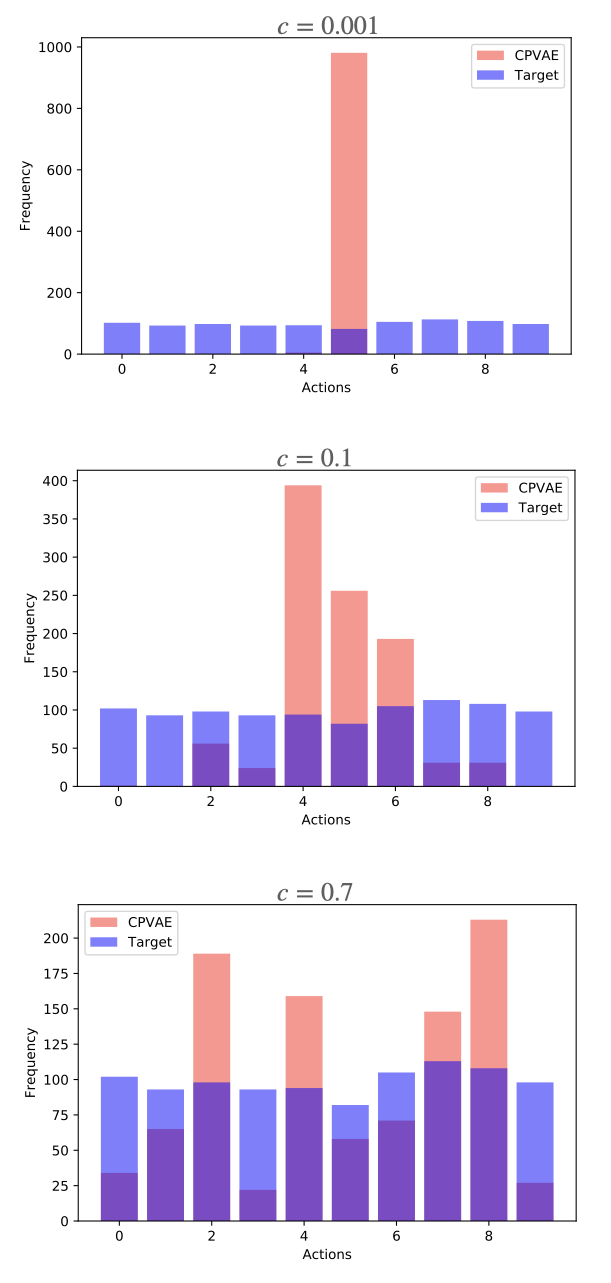}
\caption{\rev{Distribution of actions recommended by conservative strategy with $c=0.001$, $c=0.1$, and $c=0.7$ from top to bottom.}}
\label{fig:dist}
\end{figure}
\subsection{IHDP} 
In this section we repeat the experiment in \cite{mayer2020missdeepcausal} on IHDP dataset. IHDP is a semi-synthetic datasets based on the Infant Health and Development Program (IHDP) compiled by Hill \cite{hill2011}. This experiment studies the effects of specialist home visits on future cognitive test scores. The dataset comprises 25 attributes for each instance and 747 instances in total (139 treated, i.e., $a=1$, and 698 instances with $a=0$). 
Following \cite{mayer2020missdeepcausal} we report the in-sample mean absolute error in the estimation of Average Treatment Effect (ATE). ATE denoted by $\tau$ is defined as follows:
$$\tau=\mathbb{E} [R(1)-R(0)]=\mathbb{E}[\mathbb{E} [R(1)-R(0)|\Xt]].$$
Since both values of $R(0)$ and $R(1)$ for all $\Xb$'s are known from the dataset we can calculate the mean absolute error exactly $\Delta= \left\vert \hat \tau- \frac{1}{n}\sum_i R(1)_i-R(0)_i\right\vert$.
We consider scenario ``B'' of \cite{hill2011}, where $R(0)\sim \mathcal N (\mu_0,1)$ and \rev{$R(1)\sim \mathcal N (\mu_1,1)$}. Here $(\mu_0,\mu_1)=(\text{exp}(X+A)\beta,X\beta-\omega)$, $\omega$ is chosen such that we have an average treatment effect of $\tau=4$. The missing values are added with Missing Completely At Random (MCAR) mechanism. We compare three missing rates of $10\%$, $30\%$, and $50\%$.
We compare our results with several recent methods in Table \ref{IHDP-table}. MI is \rev{the} multiple imputation approach suggested in \cite{mattei2009estimating,seaman2014inverse} with 20 imputations. MF is the matrix factorization method introduced in \cite{kallus2018causal}, and MDC.process and MDC.mi are two methods introduced in \cite{mayer2020missdeepcausal}. They use a VAE to produce a latent space. Then for finding $\tau$ they fit an estimator on the latent variable. 
For all three above methods the result of an OLS estimator and two different doubly robust estimators are reported in Table 1 of \cite{mayer2020missdeepcausal}. 
Here we only report the best of the three results for each of the methods for each setting and refer to \cite{mayer2020missdeepcausal} for the complete table. MIA.GRF is a doubly robust estimator suggested in \cite{mayer2020doubly}. Finally, CEVAE a method introduced in \cite{louizos2017causal}, is another baseline we compare with. This method is not designed to work with missing values, thus a mean imputation method was performed to get the complete features before applying this method for estimating ATE. For a more detailed explanation about these competitor methods we refer to \cite{mayer2020missdeepcausal}[Section 4].

In Table \ref{IHDP-table}, for SPVAE and CPVAE we use \rev{the} maximum expected reward strategy with 5 times imputations. We can see that both of our methods outperform other methods in $50\%$ missingness and also SPVAE has the best result for $30\%$ (and comparable result in $10\%$). In the appendix we provide a table where we show our result is not sensitive to the choice of hyperparameters. 

\begin{table}[t!]
\caption{Mean absolute error with standard error for estimation of ATE for various missing rates on IHDP benchmark data. \vspace{1mm}}\label{IHDP-table}
\centering
\begin{tabular}{cccc} 
\toprule
  & 10\% &30\% & 50\%    \\
\midrule   
MI  &0.16 $\pm$ 0.00 &0.30 $\pm$ 0.00 & 0.42 $\pm$ 0.01 \\
MIA.grf &0.23 $\pm$ 0.01&0.17 $\pm$ 0.01 &0.19 $\pm$ 0.01\\
MF   &0.15 $\pm$ 0.01 & 0.16 $\pm$ 0.01  &0.20 $\pm$ 0.01  \\
CEVAE & 0.31 $\pm$ 0.01&0.38 $\pm$ 0.02&0.38 $\pm$ 0.02  \\
MDC.pro &0.15 $\pm$ 0.01 &0.15 $\pm$ 0.02&0.20 $\pm$ 0.01 \\
MDC.mi & \textbf{0.13 $\pm$ 0.02}&0.13 $\pm$ 0.01&0.18 $\pm$ 0.01  \\
SPVAE & \textbf{0.14 $\pm$ 0.01} & \textbf{0.10 $\pm$ 0.01} & \textbf{0.11 $\pm$ 0.01} \\
CPVAE  & 0.18 $\pm$ 0.01 & 0.17 $\pm$ 0.01& 0.14 $\pm$ 0.02 \\
\bottomrule 
\end{tabular}
\vspace{-5pt}
\end{table}
\subsection{Type 1 Diabetes OhioT1DM data}
For this experiment, we are using OhioT1DM dataset \cite{marling2018ohiot1dm} which contains continuous glucose monitoring, insulin dosage, physiological sensor, and self-reported life-event data for six patients with type 1 diabetes for eight weeks. Patients receiving insulin therapy are exposed to risk of hyperglycemia and hypoglycemia due to underdosing and overdosing. Therefore it is important that they receive the right dosage of bolus insulin. Note that for evaluating a recommendation method we need to have access to counterfactuals which are not available, hence, it is not possible to directly use the dataset. Here, we follow \cite{turgay2020exploiting}, and first use the dataset to train a simulator using gradient boosting and then use the trained model to produce glucose level for a pair of context (feature) and action (bolus insulin dosage). The corresponding reward will be computed using \eqref{eq:rew:ohio}. There are nine attributes for each patient and ten actions uniformly chosen between 0 and 1 (corresponding to normalized insulin dosage). To create missingness, we erase each attribute independently with probability $0.3$. We refer to appendix for a more detailed explanation of the experiment setting. We compare our method with several baselines including logistic regression (LR) and random forest (RF). In LR1 and RF1 we consider the action as one of the attributes whereas in LR2 and RF2 we train 10 different models corresponding to each of the actions. Many of the more recent competing methods accommodate only two actions and hence cannot be directly used in our setting. Here we compare with GANITE \cite{ganite2018} a GAN-based method which does not have a restriction on the number of actions. For these baseline methods we first impute missing values using both mean imputation and PVAE (we only report the best performance of the two, and use multiple imputations when using PVAE) and then feed the completed feature to the algorithm. As demonstrated in Table \ref{ohio-table}, CPVAE with MER strategy outperforms other methods and \rev{the proposed} conservative strategy has fewer instances with rewards less than -2.
\be
R=\begin{cases}
\frac{x-80}{10}, \hspace{5mm}  x\leq 90 \text{ (hypoglycemia)}\\
1, \hspace{10mm} 90\leq x\leq 130\\
\frac{180-x}{50},\hspace{5mm} 130\leq x\text{ (hyperglycemia)}
\end{cases} \label{eq:rew:ohio}
\ee
\section{Conclusion and Future work}
In this work we proposed using \rev{a} conservative strategy for dealing with uncertainty due to missingness. We suggested two methods for counterfactual estimation in the presence of missing data using VAEs. Our methods were based on using IPS which is known to have high variance. One direction for future work is to improve the method we used for estimation of the propensity scores (e.g. see \cite{mayer2020doubly}). We assumed unconfoundedness with missing values which may not hold in some scenarios. Looking for methods for relaxing this condition is another direction for future work. 
One direction for future work is to also account for the uncertainty due to the variance of our reward estimation which could vary for different actions and this can change our recommendation (we may decide to choose an action with smaller variance). Also, for computing \eqref{eq:conservative} we used random sampling. The minimum in this expression can be approximated using constrained Bayesian optimization method \cite{bayesian2014}.
\begin{table}[t!]
\caption{Average reward and the percentage of rewards less than $-2$ of different methods for OhioT1DM dataset  \vspace{1mm}}\label{ohio-table}
\centering
\begin{tabular}{ccc} 
\toprule
 Method & Average Reward & $R<-2$    \\
\midrule   
LR1  &-0.64 $\pm$ 0.02 & 0.112 $\pm$ 0.01  \\
LR2 &-0.62 $\pm$ 0.04 & 0.101 $\pm$ 0.01 \\
RF1   &-0.62 $\pm$ 0.02 & 0.110 $\pm$ 0.02    \\
RF2 & -0.61 $\pm$ 0.01 & 0.108 $\pm$ 0.01 \\
GANITE & -0.58 $\pm$ 0.09 & 0.094 $\pm$ 0.01 \\
\rev{SPVAE-Imp.} & -0.58 $\pm$ 0.03 & 0.087 $\pm$ 0.01  \\
\rev{SPVAE-MER}  & -0.56 $\pm$ 0.03 & 0.088 $\pm$ 0.01 \\
\rev{SPVAE-Cons. ($c=0.4$)} & -0.60 $\pm$ 0.02 & 0.081 $\pm$ 0.01 \\
CPVAE-Imp. & -0.58 $\pm$ 0.03 & 0.095 $\pm$ 0.02  \\
CPVAE-MER  & \textbf{-0.55 $\pm$ 0.02} & 0.093 $\pm$ 0.01 \\
CPVAE-Cons. ($c=0.4$) & -0.59 $\pm$ 0.02 & \textbf{0.080 $\pm$ 0.01} \\
\bottomrule 
\end{tabular}
\vspace{-5pt}
\end{table}



\vspace{0.5in}
\appendix
\subsection{Proposition 1}
\begin{proof}
First notice that since $H(a(\Xb)|\Xb)=0$, hence we have $H(a(\Xb),\Xb,\Xt)=H(\Xb,\Xt)$. We also have:
\begin{align*}
H(a(\Xb),\Xb,\Xt)=H(\Xt)+H(a(\Xb)|\Xt)+H(\Xb|\Xt,a(\Xb)).
\end{align*}
Therefore, 
$$H(\Xt,\Xb)=H(\Xt)+H(a(\Xb)|\Xt)+H(\Xb|\Xt,a(\Xb)).$$
Thus following equations hold:
\begin{align*}
H(a(\Xb)|\Xt)&=H(\Xt,\Xb)- H(\Xt)-H(\Xb|\Xt,a(\Xb))\\
&=H(\Xb|\Xt)-H(\Xb|\Xt,a(\Xb))\\
&=H(\Xb)-I(\Xb;\Xt)-H(\Xb|\Xt,a(\Xb))\\
&=I(\Xb;\Xt,a(\Xb))-I(\Xb;\Xt)\\
&=I(\Xb;a(\Xb))-I(\Xb;\Xt)+I(\Xb;\Xt|a(\Xb))\\
&=H(a(\Xb)) - (I(\Xb;\Xt) - I(\Xb;\Xt|a(\Xb))
\end{align*} 
In the above equations we used the fact that $H(\Xt,\Xb)=H(\Xt)+H(\Xb|\Xt)$ and $H(a(\Xb)|\Xb)=0$.
\end{proof}

\begin{example}
Assume that $\Xc=\{0,1\}^4$, and $X$ is uniformly distributed. The channel between $\Xb$ and $\Xt$ is an erasure channel which erases each bit independently with probability $1/2$. We have $\Ac=\{a_1,a_2\}$, and the reward is distributed as follows:
$$R\vert \xb,a_1 \sim \text{Ber}(\frac{x_1+x_2}{3}), \quad R\vert \xb,a_2 \sim \text{Ber}(\frac{x_3+x_4}{3}+0.1).$$
Therefore, if $x_1+x_2>x_3+x_4$, $a(\xb)=a_1$, otherwise $a(\xb)=a_2$. For instance, $X_i$ could be results of four different tests (which a subset of them will be available) and $A$ is the treatment assigned to the patient. In this setting we have $H(a(\Xb))=0.896$ this is because $x_1+x_2>x_3+x_4$ holds with probability of $5/16$, hence, we have $H(a(\Xb))=h_2(5/16)=0.896$ where $h_2$ is the binary entropy function. For computing $I(\Xb;\Xt)$, note that since attributes of $\Xb$ are independent and also each attribute will be erased independently we have: $$I(\Xb;\Xt)=4(H(\Xt)-H(\Xt|\Xb))=4(1.5-1)=2.$$
Finally, for computing $I(\Xb;\Xt|a(\Xb))$ note that $$I(\Xb;\Xt|a(\Xb))=H(\Xb|a(\Xb))-H(\Xb|\Xt,a(\Xb)).$$
Now for the second term we have:
\begin{align*}
    H(\Xb|a(\Xb))=& \frac{5}{16} H(\Xb|a(\Xb)=a_1)+\frac{11}{16} H(\Xb|a(\Xb)=a_2)\\
    =&\frac{5}{16}\log_2(\frac{1}{5})+\frac{11}{16}\log_2(\frac{1}{11}).
\end{align*} 
The other term can be computed similarly by considering the different cases of $\Xt$ and $a(\Xb)$ and using the symmetries for simplifications.
Thus, $H(a(\Xb)|\Xt)=0.570$. Note that, the probability that the best algorithm find $a(\Xb)$ by observing $\Xt$ is given by $$\frac{1}{2^{H(a(\Xb)|\Xt)}}.$$ 
Thus in our example, we can hope for guessing $a(\Xb)$ correctly on average in $67.3\%$ of cases.
\end{example}
\subsection{A Variation of SPVAE}
Instead of using propensity score we can estimate the reward using the following equation, by simply matching the similar actions. For each action $a\in \Ac$, define $\mathcal N_a =\{i: A_i=a\}$ to be the set of all indices with action $a$.
\be \hat\theta(\xb,a)=\sum_{i\in \mathcal N_a} w'_i\, R_i, \hspace{3mm} \text{where} \hspace{3mm} w'_i=\frac{p(\xb|\Xt_i)}{\sum_{j\in \mathcal N_a} p(\xb|\Xt_j)}.\label{eq:stat1}\ee
The idea is similar to original SPVAE, we estimate the reward with a weighted average of the reward of all instances for which action $a$ was prescribed. The weights measure the similarity between $x$ and $\Xt_i$.
Comparing this estimator with Eq. (12) of the paper, we note that the denominator of $w_i$ here is different and also the inverse propensity score is missing. Note that we have $\sum_{i=1}^n w_i=1$ in (12). However, on average only $\pi_0(a|\Xt_i)$ fraction of samples satisfy $\mathbb{1}[A_i=a]$.  One can interpret the propensity score in eq (12) as a way for compensate this. Basically we have 
\rev{\be\mathbb{E} \left[\sum_{i=1}^n w_i\frac{\mathbb{1}[A_i=a]}{\pi_0(a|\Xt_i)}\right]=1.\ee}
\subsection{Experiments}
Here we are outlining details of experiment setting, including hyperparameters and architecture of the neural nets. \rev{To implement our experiments, we have used JADE, the UK Tier-2 HPC Server specialised for deep learning applications. In particular, all our models are trained with a Nvidia Tesla V100 GPU card, with access to 70GB memory (though we did not use up the whole memory space). One can also run experiments 1 and 2 on a personal laptop. Also, tensorflow 1.15 is the library used for implementation.}
\subsubsection{MNIST}
For our MNIST experiment, we used a reduced MNIST dataset of 5000 data points, and performed a 80\%-20\% train-test split. We used the CPVAE architecture. The encoder followed the PNP-based architecture from \cite{ma2018eddi}. 
We used 400 dimensional feature mapping $h$ parameterized by a single fully connected network with ReLU activations, and 20 dimensional ID $\mathbf{e_i}$ for each variable. For Gaussian latent variables we used a 20 dimensional diagonal vector to represent it. The encoder (denoted by function $f$ in the paper) is a $k$-500-200-40, where $k$ is a vector resulted from the concatenation of $h$ and $a_{one}$, a one-hot encoded action vector. The network $f$ is a fully connected neural network with ReLU activations. The decoder (generator) shares the similar architecture: $Z$-200-500-$D$, where $Z$ is a vector resulted from the concatenation of the latent variable and $a_{one}$, thus here $Z=30$. Also, $D$ represents the output of the generator model which should produce pixels and the reward, hence $D=785$. For the conservative strategy, we generated $50$ random samples (with the notation of paper  $u=50$). 
During the training phase, we created artificial missingness to dataset by randomly erasing $50\%$ of the pixels from each image. We used an Adam optimizer with default hyperparameter settings, a learning rate of $0.001$, and a batch size of 8. We trained the network for 20 epochs and repeated the experiment for 100 times to get our results.
\subsubsection{Additional Experiment Results}
In Fig. \ref{fig:appedix:fig} the distribution of the rewards for MNIST experiment for three cases of conservative strategy with $c=0.001$ and $c=0.4$ and also imputation strategy is provided. As expected, this illustrates that imputation strategy has a longer tail in comparison with conservative strategies. Also, $c=0.001$ has the shortest tail.
\begin{figure}[t!]
\centering
\includegraphics[scale=0.5]{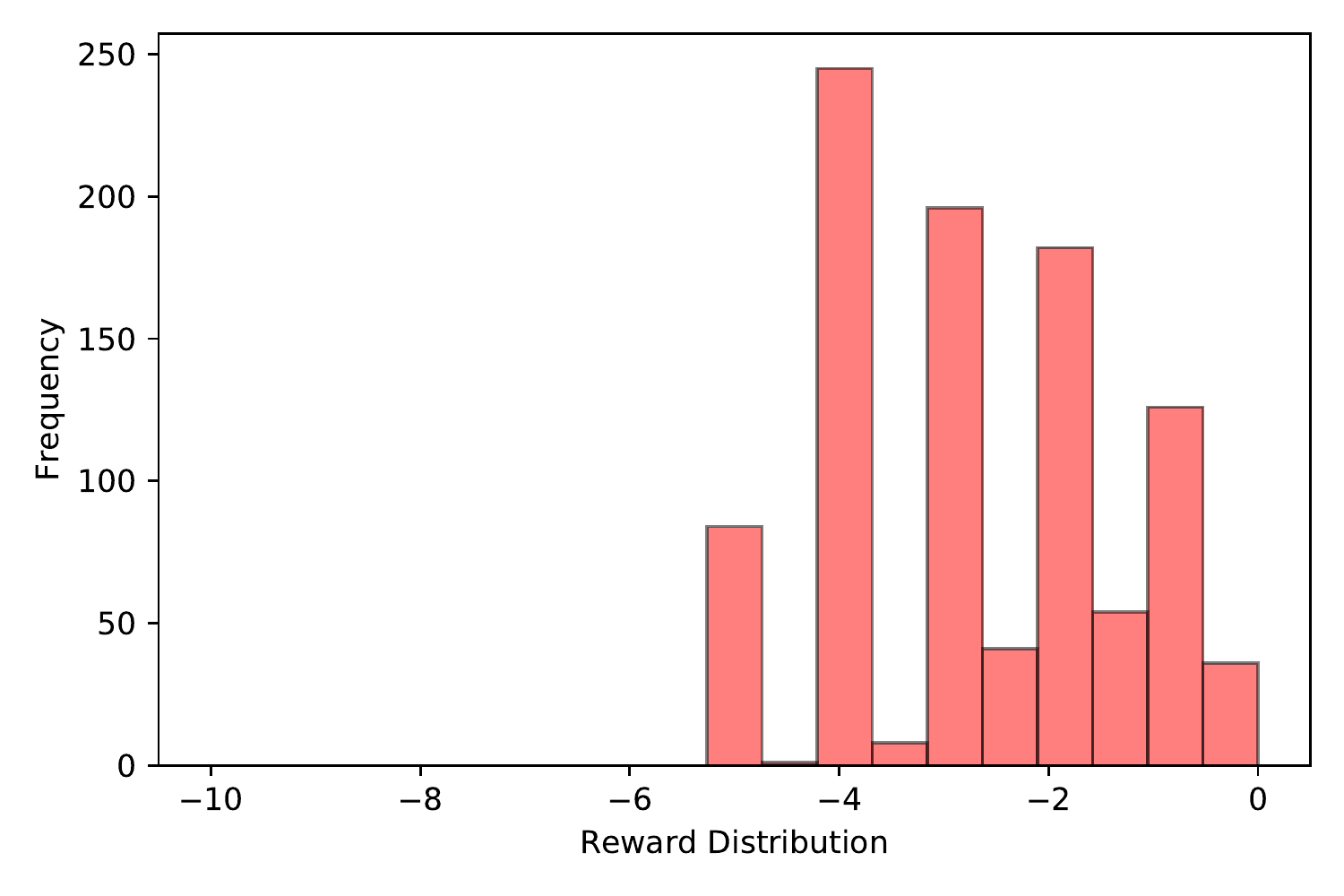} 
\includegraphics[scale=0.5]{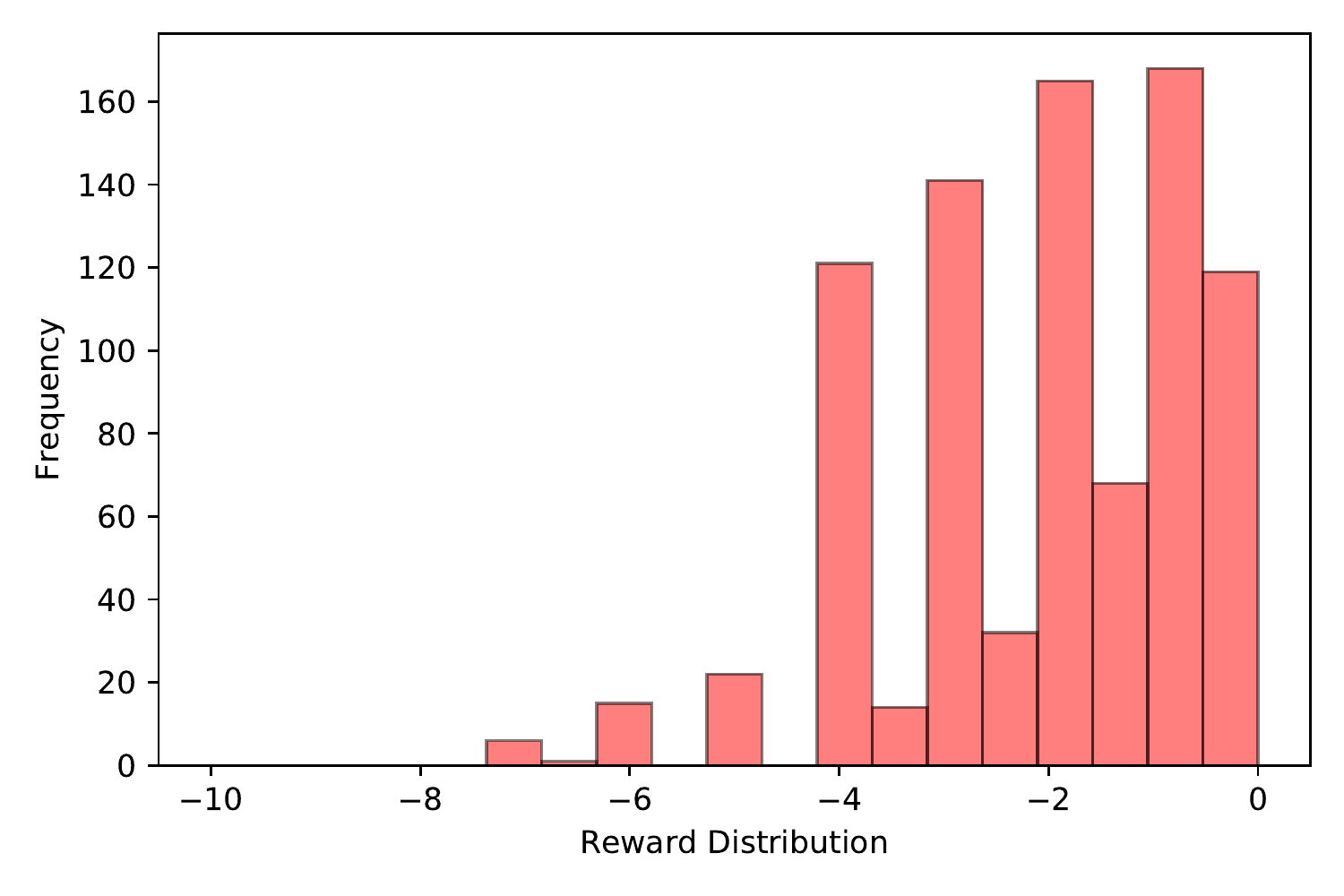} 
\includegraphics[scale=0.5]{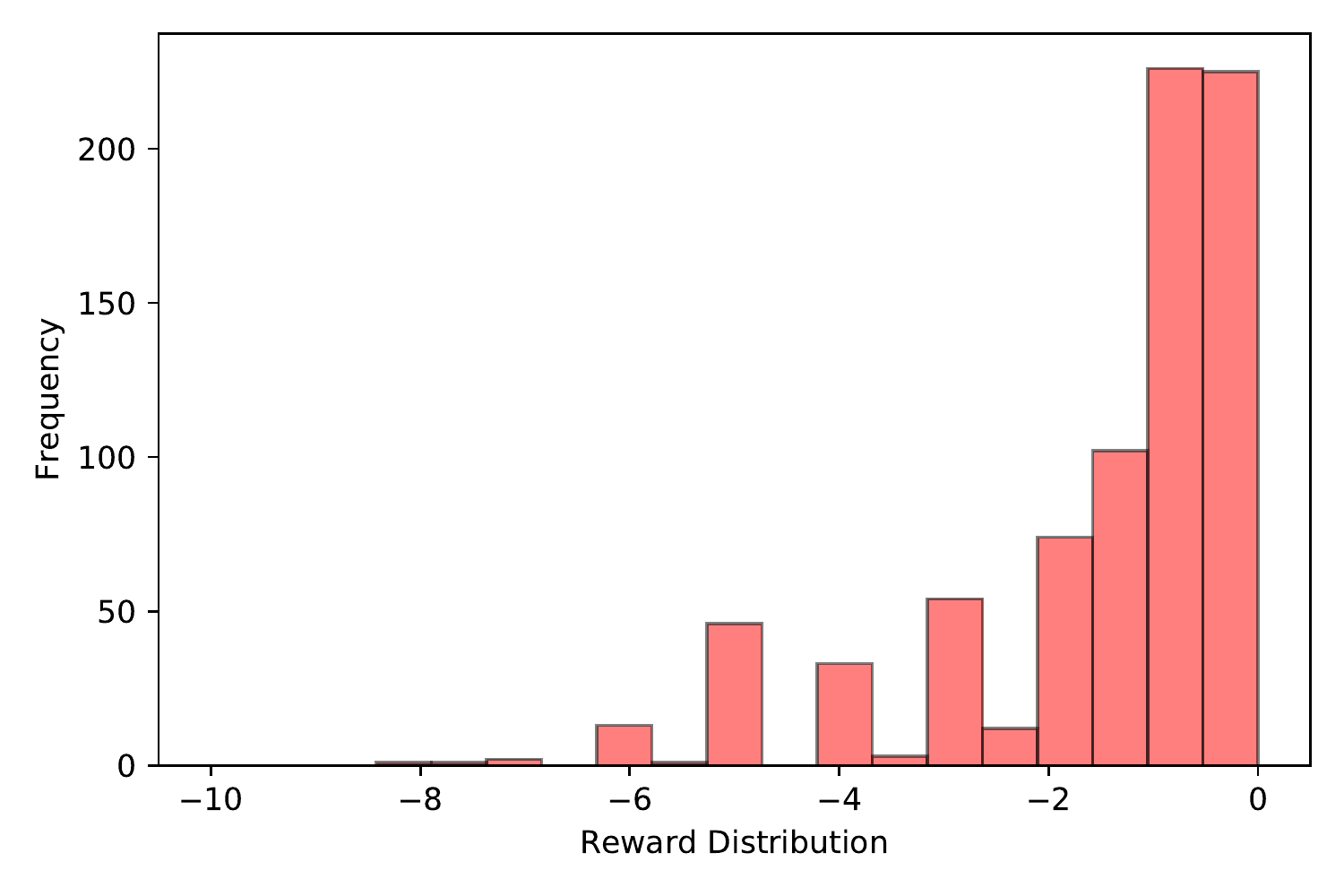} 
\caption{The distribution of reward for conservative strategies with $c=0.001$, $c=0.4$, and imputation strategy, respectively (from top to down).}
\label{fig:appedix:fig}
\end{figure}
\subsubsection{IHDP}
 We used both SPVAE and CPVAE on this dataset, for both models we used 20\% of the whole dataset as our training data. Missingness was injected to the dataset by assuming MCAR mechanism at different level of missingness, i.e. 10\%, 30\% and 50\%.

We used 5 dimensional feature mapping $h$ parameterized by a single fully connected network with ReLU activations, and 10 dimensional ID $\mathbf{e_i}$ for each variable. The Gaussian latent variable $z$ is set to a 10 dimensional diagonal vector. The inference net is a $h$-20-20-20 fully connected network with ReLU activations. The generator net is a $z$-20-20-D (where D is the observed feature dimension of IHDP dataset) fully connected network with ReLU activations. We trained the PVAE using the Adam optimizer with its default hyperparameter settings, a learning rate of 0.001 and batch size of 8. The network was trained for 25 epochs each time and the entire procedure is repeated 100 times for each missingness level.

For CPVAE, we used the same architecture as SPVAE, and the training objective is to reproduce all the attributes with their corresponding rewards given the missing attributes and the action taken. The only difference is that one-hot encoded action was added as the input of encoder and also as an input to the decoder, similar to what we described for MNIST dataset.
\subsubsection{Hyperparameters}
Here in Table \ref{MNIST-hyper}, we show that the dependence of our result for IHDP to hyperparameters is insignificant and our method consistently outperform competitors regardless of choice of hyperparameters. Here we consider several combinations of latent dimension (LD), batch size (BS), and value $K$ in the structure of PVAE which is output of the layer which encodes input variables and feed it to the first encoder.
\begin{table}[t!]
\caption{Estimated ATE for different hyperparameter setting  with $p_{miss}=0.5$. We change latent variable dim., batch size, and variable K in PVAE encoder.\vspace{1mm}}
\centering
\label{MNIST-hyper}
\begin{tabular}{cc} 
\toprule
Hyperparameters &  ATE    \\
\midrule   
LD8,BS8,K20 &$[0.1135207]$\\
LD10,BS8,K5 &$[0.11411935]$\\
LD10,BS8,K10 &$[0.11438434]$\\
LD5,BS8,K10 & $[0.1149257]$\\
LD5,BS8,K5 & $[0.11494357]$\\
LD5,BS4,K20 & $[0.11637331]$ \\ 
LD10,BS2,K10 & $[0.1164837]$ \\
LD10,BS8,K20 & $[0.11746769]$ \\
LD5,BS2,K20 & $[0.11749744]$\\
LD10,BS2,K20 & $[0.11766106]$\\
LD8,BS8,K5 & $[0.11910977]$\\
LD8,BS4,K20 & $[0.11929849]$\\
LD8,BS4,K10 & $[0.12004589]$\\
LD5,BS4,K10 & $[0.12054679]$\\
LD5,BS2,K5 & $[0.12068875]$\\
\bottomrule 
\end{tabular}
\vspace{-5pt}
\end{table}
\subsubsection{OhioT1DM}
The OhioT1DM dataset contains 8 weeks information about 6 individuals with Diabetes 1 in a time series format. Since in the original dataset only the response to the actual dose that was administered exists, it is not possible to evaluate recommendation methods directly using dataset. Thus, we use a simulator suggested in \cite{turgay2020exploiting}, that is trained on the actual data to estimate the response to a bolus injection. The simulator maps a pair of context and action to the mean of continuous glucose monitoring (CGM). From CGM the reward can be calculated according to equation (16) in the paper. A Gradient Boosting regression model with Huber loss is used to achieve this. In the model, 100 trees of maximum depth of 5 is used as weak learner. Furthermore a multivariate Gaussian distribution is fitted to approximate patients features. For producing dataset, we first sample from this distribution to get the features, then we choose an action for this feature and then feed the pair of action and feature to the simulator to produce the output. Then we randomly remove $30\%$ of features and store the triple of feature (with missing values), action, and the simulated reward in the dataset. We have produced 5000 samples for training. The way we produce actions is to train a simple logistic regression model to learn the action that is prescribed in the original dataset, and use this model to produce actions. In order to guarantee the second condition of the assumptions in Section \ref{sec:problem} (i.e. $\pi(a|\xt)>0$), we choose the action half of times from the action generator (logistic regression model) and for the other half we randomly choose one of the 10 actions. In the test time, we sample from Gaussian distribution to get the features, then randomly remove $30\%$ of the features and feed it to our method to get the recommended action. Then we feed the complete feature and action to the simulator to get the reward. We refer to \cite{turgay2020exploiting} for more details about the data generation mechanism. We chose the following hyperparameters for the model: $K=8$, $e_i=5$, latent dimension is 5, and batch size is 8. A two layer encoder and a two layer decoder is used, with 10-10-5 and 5-10-10 nodes respectively.

\bibliography{ref}
\bibliographystyle{ieeetr}




\end{document}